\documentclass[submission,copyright,creativecommons]{eptcs}
\providecommand{\event}{CREST 2019} % Name of the event you are submitting to
\usepackage{breakurl}             % Not needed if you use pdflatex only.
\usepackage{underscore}           % Only needed if you use pdflatex.

\usepackage{amssymb}
\usepackage{amsthm}

\usepackage{amsmath} \usepackage{tikz-cd} \usepackage {tikz}
\usetikzlibrary{positioning,graphs,calc,decorations.pathmorphing,shapes,arrows.meta,arrows,shapes.misc}
\usetikzlibrary{fit}
\usepackage{tkz-graph}
\usetikzlibrary{arrows}
\usepackage{comment}

\newcommand{\defi}{\stackrel{\mbox{\upshape\tiny def}}{=}}
\newtheorem{theorem}{Theorem}
\newtheorem{definition}{Definition}
\newtheorem{axiom}{Axiom}

\title{Towards A Logical Account of Epistemic Causality}
\author{Shakil M. Khan \qquad\qquad Mikhail Soutchanski
\institute{Department of Computer Science\\ Ryerson University \\ Toronto, Canada}
\email{\quad \{shakilmkhan,mes\}@scs.ryerson.ca}
}
\def\titlerunning{Towards A Logical Account of Epistemic Causality}
\def\authorrunning{S.M. Khan \& M. Soutchanski}
\begin{document}
\maketitle

\begin{abstract}
Reasoning about observed effects and their causes is important in multi-agent 
contexts. While there has been much work on causality from an 
objective standpoint, causality from the point of view of some particular 
agent has received much less attention. In this paper, we address this 
issue by incorporating an epistemic dimension to an existing formal model 
of causality. We define what it means for an agent to know the causes of an 
effect. Then using a counterexample, we prove that epistemic causality is a 
different notion from its objective counterpart.
\end{abstract}
%~~~~~~~~~~~~~~~~~~~~~~~~~~~~~~~~~~~~~~~~~~~~~~~~~~~~~~~~~~~~~~~~~~~~~~~~~~~~~~~~~~~~~~~~~~~~~~~~~~~~~~~~~~~~~~~~~~~~~~~~~~~~~~~~~~~~~
%~~~~~~~~~~~~~~~~~~~~~~~~~~~~~~~~~~~~~~~~~~~~~~~~~~~~~~~~~~~~~~~~~~~~~~~~~~~~~~~~~~~~~~~~~~~~~~~~~~~~~~~~~~~~~~~~~~~~~~~~~~~~~~~~~~~~~
%~~~~~~~~~~~~~~~~~~~~~~~~~~~~~~~~~~~~~~~~~~~~~~~~~~~~~~~~~~~~~~~~~~~~~~~~~~~~~~~~~~~~~~~~~~~~~~~~~~~~~~~~~~~~~~~~~~~~~~~~~~~~~~~~~~~~~
\section{Introduction}
Research on actual causality involves finding in a given narrative (trace) the event that caused an effect. Pearl \cite{Pearl98,Pearl00} was a pioneer to lead a computational enquiry in actual causality. 
The research was later continued by Halpern and Pearl \cite{Halpern00,HalpernP05} and others \cite{EiterL02,Hopkins05,HopkinsP07,Halpern15,Halpern16}. 
Unfortunately, as argued by Glymour et al. \cite{GlymourDGERSSTZ10}, most of these accounts are developed by analyzing a handful of simple examples, and then validated relative to our intuition for 
these examples, a process which G{\"{o}}{\ss}ler et al.\ \cite{GosslerSS17} referred to as TEGAR (i.e.\ Textbook Example Guided Analysis Refinement). As such, even after multiple revisions, these definitions continue to 
suffer from various conceptual problems such as the early preemption problem and the over-determination problem. For instance, despite claims to the contrary, the definitions given in \cite{Halpern16} suffer 
from the problem of preemption, which occurs when two competing events try to achieve the same effect and the latter of these fails to do so as the earlier one has already achieved the effect (see \cite{Weslake15} and \cite{BeckersV18} for a discussion).
\par
In an attempt to address these issues, Batusov and Soutchanski \cite{BatusovS17,BatusovS18} recently proposed a new definition of actual causality that is based on a well developed 
and expressive formalization of actions and change, namely the situation calculus \cite{McCarthyH69,Reiter01}. The definition is derived from first principles and does not follow a TEGAR scheme. 
One of the advantages of their work is that it allows one to reason about actual causes of \emph{quantified} effects. 
As argued in \cite{BatusovS18}, their definition (a version of which can be found in Section 3 below) does not suffer from preemption and can handle the more problematic examples well; 
e.g.\ both the disjunctive and the conjunctive versions of the well-known ``Forest Fire'' example \cite{HalpernP05,Halpern16} are properly handled. In our previous work \cite{KhanS18}, we showed 
that this formalization of actual causality has some intuitive properties. In particular, we proved that the computed causes of any given effect and the ``causal chain'' (relative to a 
``causal setting'', as defined in Section 3 below) are \emph{unique} for any given model, and \emph{sufficient}, but can be \emph{unnecessary} in the sense that removing 
the relevant events from the trace may still bring about the desired effect. The latter allows for other non-cause events --events that were previously preempted by the some relevant events-- 
to bring about the effect. We also proved that this formalization of causal analysis is modular since causal analysis can be performed by examining only the relevant subset of the system specification. 
Finally, we discussed how this definition can be utilized for further processing the reconstructed event traces obtained from a discrete event system (DES) based diagnoser.
\par
Note that, an important advantage of this framework is that it is based on a well developed formal theory of actions, namely the situation calculus, and as such it automatically inherits many of the 
advantages of the underlying framework. In this paper, we present our work in progress on extending 
the notion of actual causality with one such aspect, namely previous work on knowledge within the situation calculus. This allows for a first person's perspective of causality, i.e.\ causality relativised to the mental states of an agent, specially to that of her knowledge.\footnote{Handling 
belief revision complicates the framework somewhat, and therefore we focus on knowledge rather than belief.} 
Equipped with such technical machinery, agents can then reason about the causes of change in each other's mental states. This reasoning ability can be useful in distributed systems, 
be they systems of interacting agents or networked hosts, where each subsystem/agent has to take individual actions and engage in communication with other subsystems/agents. 
Finally, we envision using this formalism for developing, among other things, notions such as trust, moral responsibility, and blameworthiness within serious games/multi-agent settings. 
\par
The main contribution of this paper is two-fold. First, we incorporate a notion of knowledge within an expressive formal framework for causal analysis. To this end, we define a notion of knowledge 
of the actual causes of an observed effect, i.e.\ knowledge relative to a \emph{causal setting} (see below). This unleashes the power of causal analysis by allowing agents to reason about the causes 
of observed effects, including each other's knowledge (and goals). Secondly, using a simple counter-example, we formally show that, as expected, epistemic causality is a different notion than its 
objective counterpart in the sense that in different epistemic alternatives, different causes may bring about the same effect even if the narrative remains the same. 
\par
While doing this, we also identify a limitation of the formalism proposed in \cite{BatusovS17} and discuss how one can address this issue. 
%First, the definition in \cite{BatusovS17} implicitly assumes that the agent's initial knowledge of the domain is complete. We address this by allowing incomplete knowledge and identifying the conditions 
%under which the agent can be said to know the causes of an effect. Secondly, it 
To be specific, the definition in \cite{BatusovS17} assumes that all actions are fully observable. While we do not solve this issue in this paper, we discuss how this constraint can be relaxed by 
incorporating belief and belief revision instead of knowledge within this framework. 
\par    
The paper is organized as follows. In the next section, we outline the situation calculus. In Section 3, we give a version of the definition of actual achievement causes proposed by Batusov and 
Soutchanski \cite{BatusovS18}. In Section 4, we review previous work on knowledge in the situation calculus. Then in Section 5, we propose a model of epistemic causality and using an example show how this notion differs from the original notion of causality. Finally, we summarize our results and conclude in Section 6.
%~~~~~~~~~~~~~~~~~~~~~~~~~~~~~~~~~~~~~~~~~~~~~~~~~~~~~~~~~~~~~~~~~~~~~~~~~~~~~~~~~~~~~~~~~~~~~~~~~~~~~~~~~~~~~~~~~~~~~~~~~~~~~~~~~~~~~
%~~~~~~~~~~~~~~~~~~~~~~~~~~~~~~~~~~~~~~~~~~~~~~~~~~~~~~~~~~~~~~~~~~~~~~~~~~~~~~~~~~~~~~~~~~~~~~~~~~~~~~~~~~~~~~~~~~~~~~~~~~~~~~~~~~~~~
%~~~~~~~~~~~~~~~~~~~~~~~~~~~~~~~~~~~~~~~~~~~~~~~~~~~~~~~~~~~~~~~~~~~~~~~~~~~~~~~~~~~~~~~~~~~~~~~~~~~~~~~~~~~~~~~~~~~~~~~~~~~~~~~~~~~~~
\section{The Situation Calculus}
The situation calculus \cite{McCarthyH69} is a popular formalism for modeling and reasoning about dynamic systems. Here, we use a  
version as described by Reiter \cite{Reiter01}. There are three basic sorts in the language, \emph{situation}, \emph{action}, 
and a catch-all \emph{object} sort. A situation represents a sequence of actions. A special constant $S_0$ is used to denote the initial situation where no actions has yet been performed. 
Here, and subsequently, we use lower-case arguments for variables and upper-case arguments to represent constants. However, function and predicate symbols start with lower-case letters. 
There is a distinguished binary function symbol $do$, where $do(a, s)$ denotes the successor situation to $s$ resulting from performing the action $a$. 
For example, if $drive(agt,i,j)$ stands for an autonomous driving agent $agt$'s action of driving the car from point $i$ to point $j$, then the situation term $do(drive(Agt_1,I_1,J_1), S_0)$ denotes the situation resulting from $Agt_1$'s driving the car from $I_1$ to $J_1$ when the world is in situation $S_0$. Also, $do(drive(Agt_1,J_1,K_1), do(turn(Agt_1,J_1), do(drive(Agt_1,I_1,J_1), S_0)))$ is a situation denoting the world history consisting of the following sequence of actions:
$[drive(Agt_1,I_1,J_1), turn(Agt_1,J_1),$ $drive(Agt_1,J_1,K_1)].$ 
Thus the situations can be viewed as branches in a tree, where the root of the tree is 
$S_0$ and the edges represent actions. $do([a_1,\cdots,a_n],s)$ is used to denote the complex situation term obtained by consecutively performing 
$a_1,\cdots,a_n$ starting from $s$. Also, the notation $s\sqsubset s'$ means that situation $s'$ can be reached from situation $s$ by executing 
a sequence of actions. $s\sqsubseteq s'$ is an abbreviation of $s\sqsubset s'\vee s= s'.$ Relations whose truth values vary from situation to 
situation are called relational fluents, and are denoted by predicate symbols taking a situation term as their last argument. There is a special 
predicate $Poss(a,s)$ used to state that action $a$ is possible in situation $s$. Finally, a situation $s$ is called \emph{executable} if every 
action in its history was possible in the situation where it was performed:%\par 
%{\centering $
\[executable(s)\defi\forall a',s'.\;do(a',s')\sqsubseteq s\rightarrow Poss(a',s').\]%$\par}
\par
Following Reiter, we use a basic action theory (\textbf{BAT}) $\mathcal{D}$ that includes the following set of axioms: (1) action precondition axioms 
$\mathcal{D}_{apa}$, one per action $a$ characterizing $Poss(a, s)$, (2) successor-state axioms $\mathcal{D}_{ssa}$, one per fluent, that succinctly 
encode both effect and frame axioms and specify exactly when the fluent changes, (3) initial state axioms $\mathcal{D}_{S_0}$ describing what is true 
in $S_0$, (4) unique name axioms for actions $\mathcal{D}_{una}$, and (5) domain-independent foundational axioms $\Sigma$ describing the structure of situations.
%
%~~~~~~~~~~~~~~~~~~
%~~~~~~~~~~~~~~~~~~
%~~~~~~~~~~~~~~~~~~
%\newline\noindent\textbf{Example.}
\subsection*{Example}
We use a simple autonomous/driverless car domain as our running example. We have at least one such car/agent, $C$. %and thus henceforth we suppress the agent argument. 
An agent $c$ can drive from intersection $i$ to intersection $j$ (and turn at intersection $i$) by executing the $drive(c,i,j)$ (and $turn(c,i)$, resp.) action.\footnote{For brevity, we ignore the turn direction, the traffic light requirements, etc., although we could have easily modeled these.} The geometry of the intersections is captured using the non-fluent relation $connected(i,j),$ which states that there is a street from intersection $i$ to $j$. Unfortunately, due to poor design choices, the agents are vulnerable to over-the-air attacks by hackers. In particular, an agent $c$'s Turn Collision Avoidance System (T-CAS) can be easily corrupted by executing the $hack(c)$ action. If its T-CAS is corrupted, turning a car damages it. Finally, initially all the cars are undamaged and their T-CAS are uncorrupted. 
\par
There are three fluents in this domain, $at(c,i,s)$, $corrupted(c,s)$, and $damaged(c,s)$, which mean that the agent $c$ is at location $i$ in situation $s$, $c$'s T-CAS is corrupted in $s$, and $c$ is damaged in $s$, respectively. 
\par
We now give the domain-dependent axioms specifying this example domain. First, the preconditions for $drive(c,i,j)$, $turn(c,i)$, and $hack(c)$ can be specified using action precondition axioms (\textbf{APA}) as follows (henceforth, all free variables in a sentence are assumed to be universally quantified): 
%\vskip-\parskip
%\vspace{-4 mm}
%{\footnotesize
\begin{eqnarray*}
&&\hspace{-7 mm}(a).\;Poss(drive(c,i,j),s)\leftrightarrow at(c,i,s)\wedge i\neq j \wedge connected(i,j),\\
&&\hspace{-7 mm}(b).\;Poss(turn(c,i),s)\leftrightarrow at(c,i,s),\\
&&\hspace{-7 mm}(c).\;Poss(hack(c),s).
\end{eqnarray*}
%}
That is, $(a)$ an agent $c$ can drive from intersection $i$ to $j$ in some situation $s$ if and only if $c$ is at intersection $i$ in situation $s$, $i$ and $j$ refer to different intersections, and there is a street connecting $i$ and $j$; 
$(b)$ $c$ can turn at intersection $i$ in situation $s$ if and only if $c$ is at $i$ in $s$; and $(c)$ a hacker can always hack $c$.   
\par
Moreover, the following successor-state axioms (\textbf{SSA}) specify how exactly the fluents $at$, $corrupted$, and $damaged$ changes value when an action $a$ happens in some situation $s$:
%\vskip-\parskip
%\vspace{-4 mm}
%{\footnotesize
\begin{eqnarray*}
&&\hspace{-7 mm}(d).\;at(c,i,do(a,s))\leftrightarrow %(\neg\exists b'.\;holding(b',s)\land 
(\exists j(a=drive(c,j,i))%)\\&&\hspace{33 mm}\mbox{}
\lor (at(c,i,s)\land\neg\exists j(a=drive(c,i,j)))),\\
&&\hspace{-7 mm}(e).\;corrupted(c,do(a,s))\leftrightarrow (a=hack(c)\lor corrupted(c,s)),\\
&&\hspace{-7 mm}(f).\;damaged(c,do(a,s))\leftrightarrow ((corrupted(c,s)\land\exists i(a=turn(c,i)))\lor damaged(c,s)).
\end{eqnarray*}
%}
That is, $(d)$ an agent $c$ is at location $i$ in the situation resulting from executing some action $a$ in situation $s$ (i.e.\ in $do(a,s)$) if and only if $a$ 
refers to $c$'s action of driving from location $j$ to $i$, 
or she was already at $i$ in $s$ and $a$ is not the action of her driving to another location $j$; 
$(e)$ $c$'s T-CAS is corrupted after action $a$ happens in situation $s$ if and only if $a$ is the action of hacking $c$ or her T-CAS was already corrupted in $s$; and $(f)$ $c$ is damaged after action 
$a$ happens in situation $s$ if and only if $c$'s T-CAS was corrupted in $s$ and $a$ refers to the action of turning $c$ at some intersection $i$, or $c$ was already damaged in $s$.   
%
%The case for the fluent $broken$ in $(f)$ is similar. 
%
\par
Furthermore, the following initial state axioms say that initially $(g)$ all the agent's T-CAS are uncorrupted,  $(h)$ all agents are undamaged, and $(i)$ they are located at intersection $I$:
%\vskip-\parskip
%\vspace{-4 mm}
%{\footnotesize
\begin{eqnarray*}
&&\hspace{-7 mm}(g).\;\forall c(\neg corrupted(c,S_0)),\hspace{10 mm}(h).\;\forall c(\neg damaged(c,S_0)),\hspace{10 mm}(i).\;\forall c(at(c,I,S_0)).%,\\
%&&\hspace{-7 mm}(h).\;\forall b.\;\neg fragile(b,S_0),\\
%&&\hspace{-7 mm}(i).\;\forall b.\;\neg broken(b,S_0).
\end{eqnarray*}
%}
% 
\indent
We assume for simplicity three intersections connected with two streets:
\begin{eqnarray*}
&&\hspace{-7 mm}(j).\;\forall i,j.\;connected(i,j)\leftrightarrow((i=I\wedge j=J)\vee(i=J\wedge j=I)\vee (i=J\wedge j=K)\vee (i=K\wedge j=J)).
\end{eqnarray*}
\indent
%zzz
\begin{comment}%CHANGE BACK AND DISCUSS WITH YVES
Finally, we assume the domain closure axiom $(k)$ for the object sort, stating that there are only three intersections $I, J,$ and $K$ and one car $C$ in the domain, and unique names axioms $(l)$ for intersections, stating that $I, J,$ and $K$ refer to three different intersections. 
\begin{eqnarray*}
&&\hspace{-7 mm}(k).\;\forall i(i=I\vee i=J\vee i=K)\wedge \forall c(c=C),\\
&&\hspace{-7 mm}(l).\;I\neq J\wedge I\neq K\wedge J\neq K.
\end{eqnarray*}
\end{comment}
%zzz_END
Also, for simplicity and illustration, we assume the domain closure axiom $(k)$ for the intersections, stating that there are only three intersections $I, J,$ and $K$ in this domain:
\begin{eqnarray*}
&&\hspace{-7 mm}(k).\;\forall i(i=I\vee i=J\vee i=K).
\end{eqnarray*}
However, we do not require a domain closure axiom for cars/agents, as their number can be unknown. 
Finally, we need unique names axioms $(l)$, stating that $I, J,$ and $K$ refer to three different intersections: 
\begin{eqnarray*}
&&\hspace{-7 mm}(l).\;I\neq J\wedge I\neq K\wedge J\neq K.
\end{eqnarray*}
Also the following unique names for actions axioms (\textbf{UNA}) say that $(m)$ $drive$, $turn$, and $hack$ refer to different actions, and $(n)$ two actions with the same function symbol refer to the same action if their arguments are the same (these are necessary for the above successor-state axioms to work properly):
%\vskip-\parskip
%\vspace{-4 mm}
%{\footnotesize
\begin{eqnarray*}
&&\hspace{-7 mm}(m).\;\forall c_1,c_2,i,j,k(drive(c_1,i,j)\neq turn(c_2,k)\wedge drive(c_1,i,j)\neq hack(c_2)\wedge turn(c_1,k)\neq hack(c_2)),\\
&&\hspace{-7 mm}(n).\;\forall c_1,c_2,i,j,k,l((drive(c_1,i,j)=drive(c_2,k,l)\rightarrow(c_1=c_2\wedge i=k\wedge j=l))\\
&&\hspace{19 mm}\mbox{}\wedge (turn(c_1,i)=turn(c_2,j)\rightarrow(c_1=c_2\wedge i=j))\wedge(hack(c_1)=hack(c_2)\rightarrow(c_1=c_2)).
%&&\hspace{-7 mm}(h).\;\forall p,p',f,f'.\;pickUp(p,f)=pickUp(p',f')\\
%&&\hspace{18 mm}\mbox{}\rightarrow p=p'\land f=f'.
\end{eqnarray*}
%}
Henceforth, we use $\mathcal{D}_{ac}$ to refer to the above axiomatization of the autonomous car domain.
%~~~~~~~~~~~~~~~~~~
%~~~~~~~~~~~~~~~~~~555
%~~~~~~~~~~~~~~~~~~
%\newline\noindent\textbf{Regression in the SC.}
\subsection*{Regression in the Situation Calculus}
BATs employ \emph{regression}, a powerful reasoning mechanism for answering queries about the future. Given a query ``does $\phi$ 
hold in the situation obtained by performing the ground action $\alpha$ in situation $s$, i.e.\ in $do(\alpha,s)?$'',\footnote{A ground term is one whose 
constituents are ground sub-terms and constants, i.e.\ that contains no variables.} the single-step regression operator $\rho$ transforms it into an 
equivalent query ``does $\psi$ hold in situation $s?$'', eliminating action $\alpha$ by compiling it into $\psi.$ The expression 
$\rho[\phi,\alpha]$ denotes such a logically equivalent query obtained from the formula $\phi$ by replacing each fluent atom $F$ in $\phi$ with the right-hand side of the 
successor-state axiom for $F$ where the action variable $a$ is instantiated with the ground action $\alpha$, and then simplified using unique name 
axioms for actions and constants. One can prove that given a BAT $\mathcal{D}$, a formula $\phi(s)$ \emph{uniform in} $s$ (meaning that it has no 
occurrences of $Poss$, $\sqsubseteq$, other situation terms besides $s$, and quantifiers over situations), and a ground action term $\alpha$, 
we have that $\mathcal{D}\models\forall s.\;\phi(do(\alpha,s))\leftrightarrow\rho[\phi(s),\alpha]$. One can also obtain a similar regression operator 
$\mathcal{R}$ by repetitive recursive application of $\rho$. Reiter \cite{Reiter01} showed that for a \emph{regressable} query $\phi$, 
$\mathcal{D}\models\phi$ if and only if $\mathcal{D}_{una}\cup\mathcal{D}_{S_0}\models\mathcal{R}[\phi]$. Regression thus simplifies 
entailment checking by compiling dynamic aspects of the theory into the query. 
%~~~~~~~~~~~~~~~~~~
%~~~~~~~~~~~~~~~~~~
%~~~~~~~~~~~~~~~~~~
%\newline\noindent\textbf{Example.}
\subsection*{Example (Continued)}
Let us compute $\rho[damaged(C,do(turn(C,K),S^*)),turn(C,K)]$, for some situation $S^*$. From the right-hand side 
of the SSA $(f)$ above and by substituting action variable $a$ by $turn(C,K)$, object variables $c$ by $C$ and $i$ by $K$, and situation variable $s$ by 
$S^*,$ the result of single-step regression $\rho[damaged(C,do(turn(C,$ $K),S^*)),turn(C,K)]$ amounts to 
$(corrupted(C,S^*)\wedge turn(C,K)=turn(C,K))\vee damaged(C,S^*).$ Using the unique names axiom $(n)$ above, the result of $\rho$ can be simplified to 
$corrupted(C,S^*)\vee damaged(C,S^*)$. Thus, in this example $\rho$ allows us to answer the query $damaged(C,do(turn(C,K),S^*))$ relative to situation $do(turn(C,K),S^*)$ by 
reducing it to the equivalent simpler query $corrupted(C,S^*)\vee damaged(C,S^*)$ that only mentions the preceding situation $S^*$ and does not mention the situation $do(turn(C,K),S^*)$.
%
%As an example, let us compute $\rho[broken(B_1,s),drop(B_1)]$, where $s=do(drop(B_1),s')$ for some situation $s'$. From the right-hand side 
%of the SSA $(f)$ above and by substituting action variable $a$ by $drop(B_1)$, object variable $b$ by $B_1$, and situation variable $s$ by 
%$do(drop(B_1),s'),$ the result of single-step regression $\rho[broken(B_1,s),drop(B_1)]$ amounts to 
%$(f\!ragile(B_1,s')\wedge drop(B_1)=drop(B_1))\vee broken(B_1,s').$ Using the unique names axiom $(k)$ above, the result of $\rho$ can be simplified to 
%$f\!ragile(B_1,s')\vee broken(B_1,s')$. Thus, in this example $\rho$ allowed us to answer the query $broken(B_1,s)$ relative to situation $s=do(drop(B_1),$ $s')$ by 
%reducing it to the equivalent query $f\!ragile(B_1,s')\vee broken(B_1,s')$ that only mentions the previous situation $s'$ (and does not mention the situation $s$).
%
%~~~~~~~~~~~~~~~~~~~~~~~~~~~~~~~~~~~~~~~~~~~~~~~~~~~~~~~~~~~~~~~~~~~~~~~~~~~~~~~~~~~~~~~~~~~~~~~~~~~~~~~~~~~~~~~~~~~~~~~~~~~~~~~~~~~~~
%~~~~~~~~~~~~~~~~~~~~~~~~~~~~~~~~~~~~~~~~~~~~~~~~~~~~~~~~~~~~~~~~~~~~~~~~~~~~~~~~~~~~~~~~~~~~~~~~~~~~~~~~~~~~~~~~~~~~~~~~~~~~~~~~~~~~~
%~~~~~~~~~~~~~~~~~~~~~~~~~~~~~~~~~~~~~~~~~~~~~~~~~~~~~~~~~~~~~~~~~~~~~~~~~~~~~~~~~~~~~~~~~~~~~~~~~~~~~~~~~~~~~~~~~~~~~~~~~~~~~~~~~~~~~
\section{Actual Achievement and Maintenance Causes}
Given a trace of events,\footnote{We do not conceptually distinguish between agents' actions and exogenous/nature's events.} 
\emph{actual achievement causes} are some of the events that are behind achieving an effect while \emph{actual maintenance causes} are those which are 
responsible for mitigating the threats to the achieved effect. There can be also cases of subtle interactions of these two. In this section, we review how one can define 
achievement causality in the situation calculus \cite{BatusovS18}. An effect in this framework is a situation calculus formula $\phi(s)$ that is uniform in $s$ and that may include quantifiers over object variables. 
Given an effect $\phi(s),$ the actual causes of $\phi$ are defined relative to a \emph{causal setting} that includes a BAT $\mathcal{D}$ representing the domain 
dynamics, and a ``narrative'' (a trace of events) $\sigma$, representing the ground situation, where the effect was observed. 
\begin{definition}[Causal Setting]
%\newline\noindent
%\textbf{Definition 1 (Causal Setting).} 
%\emph{
A causal setting is a tuple $\langle\mathcal{D},\sigma,\phi(s)\rangle$, where $\mathcal{D}$ is a BAT, $\sigma$ is a ground 
situation term of the form $do([a_1,\cdots,a_n], S_0)$ with ground action functions $a_1,\cdots,a_n$ such that $\mathcal{D}\models executable(\sigma)$, and $\phi(s)$ 
is a situation calculus formula uniform in $s$ such that $\mathcal{D}\models\phi(\sigma)$.
%}
%\newline\indent
\end{definition}
As the theory $\mathcal{D}$ does not change, %when referring to a causal setting 
we will often suppress $\mathcal{D}$ and simply write $\langle\sigma,\phi(s)\rangle$. 
Also, here we require $\phi$ to hold by the end of the narrative $\sigma,$ and thus ignore the cases where $\phi$ is not achieved by the actions in $\sigma$, since if 
this is the case, the achievement cause truly does not exist.   
\par
Note that since all changes in the situation calculus result from actions, we identify the potential causes of an effect $\phi$ with a set of ground action terms occurring 
in $\sigma$. However, since $\sigma$ might include multiple occurrences of the same action, we also need to identify the situations when these actions were executed. 
Now, the notion of the achievement cause of an effect suggests that if some action $\alpha$ of the action sequence in $\sigma$ triggers the formula $\phi(s)$ to change 
its truth value from false to true relative to $\mathcal{D}$, and if there are no actions in $\sigma$ after $\alpha$ that change the value of $\phi(s)$ back to false, then 
$\alpha$ is the actual cause of achieving $\phi(s)$ in $\sigma$. %The soundness of this statement follows from the fact that the actions in the narrative $\sigma$ are 
%totally ordered (as determined by $\sigma$), that change is associated with a particular action of that order, and that fluents in the situation calculus only change value as a result of actions.
\par
When used together with the single-step regression operator $\rho$, the above interpretation of achievement condition not only identifies the single action that brings 
about the effect of interest, but also captures the actions that build up to it. Intuitively, $\rho[\phi,\alpha]$ specifies the weakest condition that must hold in a previous 
situation (let us call it $\sigma'$) in order for $\phi$ to hold after performing the action $\alpha$ in situation $\sigma'$, i.e.\ in situation $do(\alpha,\sigma')$. Thus, if the 
action $\alpha$ is an achievement cause of $\phi$ in situation $do(\alpha,\sigma')$, then we can use the single-step regression operator $\rho$ to obtain a formula that 
holds at situation $\sigma'$ and constitutes a necessary and sufficient condition for the achievement of $\phi(s)$ via the action $\alpha$. This new formula may have 
an achievement cause of its own which, by virtue of the action $\alpha$, also constructively contributes to the achievement of $\phi$. By repeating this process, we 
can uncover the entire chain of actions that incrementally build up to the achievement of the ultimate effect. At the same time, we must not overlook the conditions 
that make the execution of the action $\alpha$ in situation $\sigma$ even possible, which are conveniently captured by the right-hand side of the 
action precondition axiom for $\alpha$ and may have achievement causes of their own. 
\par
%
%To summarize, if $\alpha$ achieves $\phi(s)$ in $do(\alpha,\sigma')$ and no other actions negate this effect afterwards (i.e.\ $\alpha$ is an achievement cause 
%of $\phi(s)$ in $do(\alpha,\sigma')$), then $\rho[\phi(s), \alpha]$ and the precondition of $\alpha$ express the condition which (a) holds at $\sigma'$, (b) is 
%necessary and sufficient for executing $\alpha$ in $\sigma'$, and (c) is necessary and sufficient for achieving $\phi(s)$ via $\alpha$. 
The following inductive definition formalizes this intuition. % behind achievement causality. 
Let $\Pi_{apa}(\alpha,\sigma)$ be the right-hand side of the action precondition axiom 
for action $\alpha$ with the situation term replaced by situation $\sigma$. %Then:
\begin{definition}[Achievement Cause]\label{oldd2}
%\newline\noindent
%\textbf{Definition 2 (Achievement Cause).} 
%\emph{
A causal setting $\mathcal{C}=\langle\sigma,\phi(s)\rangle$ satisfies the achievement condition of $\phi$ via the 
situation term $do(\alpha^*,\sigma^*)\sqsubseteq\sigma$ if and only if there is an action $\alpha'$ and situation $\sigma'$ such that: 
%}
%\[\mathcal{D}\models\neg\phi(\sigma')\wedge\forall s.\;do(\alpha',\sigma')\sqsubseteq s\sqsubseteq\sigma\rightarrow\phi(s),\] 
%\par{\centering $
\[\mathcal{D}\models\neg\phi(\sigma')\wedge\forall s.\;do(\alpha',\sigma')\sqsubseteq s\sqsubseteq\sigma\rightarrow\phi(s),\]
%$\par} 
%\noindent\emph{
and either $\alpha^* = \alpha'$ and $\sigma^*=\sigma'$, or 
the causal setting $\langle\sigma',\rho[\phi(s),\alpha']\wedge\Pi_{apa}(\alpha',\sigma')\rangle$ satisfies the achievement condition via the situation term 
$do(\alpha^*,\sigma^*).$
Whenever a causal setting $\mathcal{C}$ satisfies the achievement condition via situation $do(\alpha^*,\sigma^*),$ we say that the action $\alpha^*$ 
executed in situation $\sigma^*$ is an \emph{achievement cause} in causal setting $\mathcal{C}$. 
%}
%\newline\indent
\end{definition}
Since the process of discovering intermediary achievement causes using the single-step regression operator $\rho$ cannot continue beyond $S_0$, it 
eventually terminates. Moreover, since the narrative $\sigma$ is a finite sequence, the achievement causes of $\mathcal{C}$ also form a finite sequence of 
situation-action pairs, which we call the \emph{achievement causal chain of} $\mathcal{C}$. 
%Note that while we could have defined the sequence of actions as the causal chain, this might be ambiguous as there may be more than one occurrence of 
%some action in the narrative. Also, we could have modeled a causal chain using only the relevant situations if these are written as $do(\alpha^*,\sigma^*).$ 
\par
As shown in \cite{BatusovS17}, one can also define the concept of \emph{maintenance causes} by appealing to a counterfactual notion 
of potential threats in the causal setting that can possibly flip the truth value of the effect $\phi$ to false, and actions in the narrative that 
mitigated those threats. In general, actual causes can be either achievement causes or maintenance causes and the causal chain can 
include both. %Armed with this notion of causal chains, we now shift our attention to diagnosis.
However, to keep things simple, in this paper we focus exclusively on actual achievement causes. 
%~~~~~~~~~~~~~~~~~~
%~~~~~~~~~~~~~~~~~~
%~~~~~~~~~~~~~~~~~~
%~~~~~~~~~~~~~~~~~~
%~~~~~~~~~~~~~~~~~~
%~~~~~~~~~~~~~~~~~~
\begin{figure}[t!]%[h]
\setlength{\unitlength}{0.14in} % selecting unit length
\centering % used for centering Figure
\begin{picture}(44,17) % picture environment with the size (dimensions)
 % 32 length units wide, and 15 units high.
\put(2,9) {${\color[rgb]{0.9,0,0}S_0}$}
\put(0,10){\framebox(4,2){$at(I)$}}
\put(0,12){\framebox(4,2){$\neg corrpt$}}
\put(0,14){\framebox(4,2){$\neg dmgd$}}
\put(4,13){\vector(1,0){4}}
\put(4.5,14) {${\color[rgb]{0.9,0,0}drive_{IJ}}$}
\put(10,9) {$S_1$}
\put(8,10){\framebox(4,2){$at(J)$}}
\put(8,12){\framebox(4,2){$\neg corrpt$}}
\put(8,14){\framebox(4,2){$\neg dmgd$}}
\put(12,13){\vector(1,0){4}}
\put(13,14) {$turn_J$}
\put(18,9) {${\color[rgb]{0.9,0,0}S_2}$}
\put(16,10){\framebox(4,2){$at(J)$}}
\put(16,12){\framebox(4,2){$\neg corrpt$}}
\put(16,14){\framebox(4,2){$\neg dmgd$}}
\put(20,13){\vector(1,0){4}}
\put(21,14){${\color[rgb]{0.9,0,0}hack}$}
\put(26,9) {${\color[rgb]{0.9,0,0}S_3}$}
\put(24,10){\framebox(4,2){$at(J)$}}
\put(24,12){\framebox(4,2){$corrpt$}}
\put(24,14){\framebox(4,2){$\neg dmgd$}}
\put(28,13){\vector(1,0){4}}
\put(28.5,14) {${\color[rgb]{0.9,0,0}drive_{JK}}$}
\put(34,9) {${\color[rgb]{0.9,0,0}S_4}$}
\put(32,10){\framebox(4,2){$at(K)$}}
\put(32,12){\framebox(4,2){$corrpt$}}
\put(32,14){\framebox(4,2){$\neg dmgd$}}
\put(36,13){\vector(1,0){4}}
\put(37,14) {${\color[rgb]{0.9,0,0}turn_K}$} 
\put(40.5,9) {$S_5=\sigma1$}
\put(40,10){\framebox(4,2){$at(K)$}}
\put(40,12){\framebox(4,2){$corrpt$}}
\put(40,14){\framebox(4,2){$dmgd$}}
%
%%%%%
%
\put(2,0) {${\color[rgb]{0.9,0,0}S_0^*}$}
\put(0,1){\framebox(4,2){$at(I)$}}
\put(0,3){\framebox(4,2){$corrpt$}}
\put(0,5){\framebox(4,2){$\neg dmgd$}}
\put(4,4){\vector(1,0){4}}
\put(4.5,5) {${\color[rgb]{0.9,0,0}drive_{IJ}}$}
\put(10,0) {${\color[rgb]{0.9,0,0}S_1^*}$}
\put(8,1){\framebox(4,2){$at(J)$}}
\put(8,3){\framebox(4,2){$\textit{corrpt}$}}
\put(8,5){\framebox(4,2){$\neg dmgd$}}
\put(12,4){\vector(1,0){4}}
\put(13,5) {${\color[rgb]{0.9,0,0}turn_J}$}
\put(18,0) {$S_2^*$}
\put(16,1){\framebox(4,2){$at(J)$}}
\put(16,3){\framebox(4,2){$\textit{corrpt}$}}
\put(16,5){\framebox(4,2){$dmgd$}}
\put(20,4){\vector(1,0){4}}
\put(21,5){$hack$}
\put(26,0) {$S_3^*$}
\put(24,1){\framebox(4,2){$at(J)$}}
\put(24,3){\framebox(4,2){$corrpt$}}
\put(24,5){\framebox(4,2){$\textit{dmgd}$}}
\put(28,4){\vector(1,0){4}}
\put(28.5,5) {$drive_{JK}$}
\put(34,0) {$S_4^*$}
\put(32,1){\framebox(4,2){$at(K)$}}
\put(32,3){\framebox(4,2){$corrpt$}}
\put(32,5){\framebox(4,2){$\textit{dmgd}$}}
\put(36,4){\vector(1,0){4}}
\put(37,5) {$turn_K$} 
\put(42,0) {$S_5^*$}
\put(40,1){\framebox(4,2){$at(K)$}}
\put(40,3){\framebox(4,2){$corrpt$}}
\put(40,5){\framebox(4,2){$dmgd$}}
\end{picture}
\caption{Evolution of fluents relative to narrative $\sigma_1$ starting in situations $S_0$ and $S_0^1$} % title of the Figure
\label{fig:evo0} % label to refer figure in text
\end{figure}
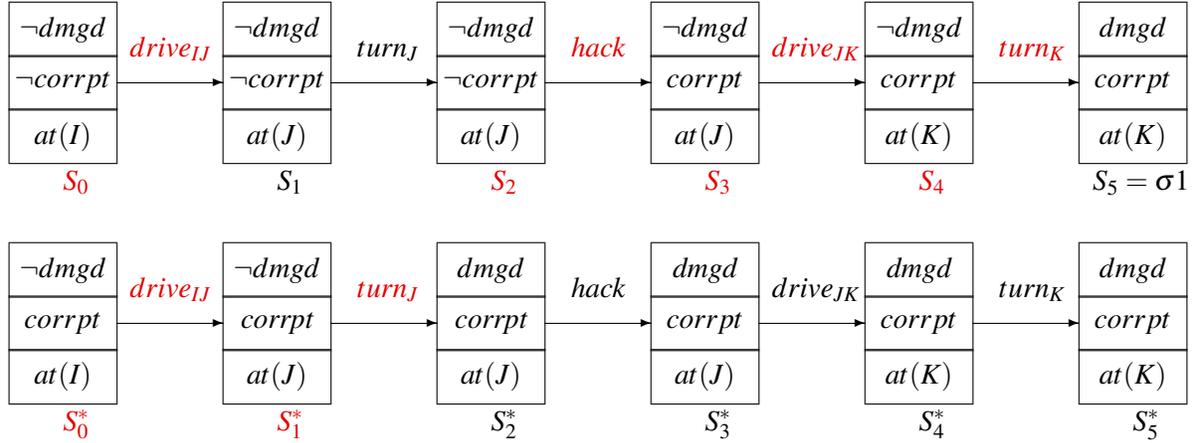
%~~~~~~~~~~~~~~~~~~
%~~~~~~~~~~~~~~~~~~
%~~~~~~~~~~~~~~~~~~
\subsection*{Example (Continued)}
Consider the narrative $\sigma_1=do([drive(C,I,J),turn(C,J),hack(C),drive(C,J,K),turn(C,K)],S_0),$ i.e.\ the agent $C$ drives the car from intersection $I$ to $J$, then $C$ turns at intersection $J$, then $C$'s T-CAS gets hacked, then $C$ drives to intersection $K$, and finally $C$ turns at $K$ (see the top part of Figure \ref{fig:evo0}). We are interested in computing the actual causes of the effect $\phi_1=damaged(C,s)$. Then according to Definition 2, the causal setting $\langle\phi_1,\sigma_1\rangle$ satisfies the achievement condition $\phi_1$ via the situation term $do(turn(C,K),S_4),$ where $S_4=do([drive(C,I,J),turn(C,J),hack(C),drive(C,J,K)],S_0),$ so the action $turn(C,K)$ executed in situation $S_4$ is a primary achievement cause of $damaged(C,s)$.
\par
Moreover, let us compute $\rho[damaged(C,s),turn(C,K)]$ and $Poss(turn(C,K),S_4)$, starting with the former. As shown in Section 2 above, the result of $\rho$ can be simplified to 
$corrupted(C,S_4)\vee damaged(C,$ $S_4)$. 
%
%Moreover, let us compute $\rho[broken(B_1,s),drop(B_1)]$ and $Poss(drop(B_1),S_3)$, starting with the former. From right-hand side 
%of the SSA $(f)$ above and by substituting action variable $a$ by $drop(B_1)$, object variable $b$ by $B_1$, and situation variable $s$ by 
%$\sigma_1=do(drop(B_1),S_3),$ the result of single-step regression $\rho[broken(B_1,s),drop(B_1)]$ amounts to 
%$(f\!ragile(B_1,S_3)\wedge drop(B_1)=drop(B_1))\vee broken(B_1,S_3).$ Using the unique names axiom $(k)$ above, the result of $\rho$ can be simplified to 
%$f\!ragile(B_1,S_3)\vee broken(B_1,S_3)$. 
%
Let us now consider $Poss(turn(C,K),S_4)$; from the right-hand side of action precondition axiom $(b)$ above and by replacing object variables $c$ with $C$ and $i$ with $K$ and situation variable $s$ by $S_4$, 
we have $at(C,K,S_4).$ 
Computing $\rho[damaged(C,s),turn(C,K)]\wedge Poss(turn(C,K),S_4)$ thus gives rise to a new causal setting 
$\langle (corrupted(C,s)\vee damaged(C,s))\wedge at(C,K,s),S_4\rangle$. As can be seen in the top part of Figure \ref{fig:evo0}, this setting satisfies the achievement 
condition via the action $drive(C,J,K)$, so $drive(C,J,K)$ executed in $S_3=do([drive(C,I,J),turn(C,J),hack(C)],S_0)$ is a secondary achievement cause. 
Notice that, while at a first glance driving the car $C$ from intersection $J$ to $K$ may not seem like an intuitive cause for the damage to the car, it can be argued that it is actually a cause. In particular, 
for $C$ to be damaged, the $turn(C,K)$ action needs to be executable in situation $S_4$. By the APA $(b)$ above, this means that $C$ must be at intersection $K$ in $S_4$, which can only be achieved by executing the $drive(C,J,K)$ action in situation $S_3$. Thus, given narrative $\sigma_1,$ $drive(C,J,K)$ indeed indirectly contributes to the car $C$'s damage.   
\par
Furthermore, this yields yet another setting: 
\[\langle\rho[(corrupted(C,s)\vee damaged(C,s))\wedge at(C,K,s),drive(C,J,K)]\wedge Poss(drive(C,J,K),S_3),S_3\rangle.\] 
Doing simplifications similar to what we did before, we can arrive at the next setting $\langle(corrupted(C,$ $s)\vee damaged(C,s))\wedge at(C,J,s),S_3\rangle$, 
which meets the achievement condition via the action $hack(C)$ executed in situation 
$S_2=do([drive(C,I,J),turn(C,J)],S_0)$.
\par
And again, this yields another setting: 
\[\langle\rho[(corrupted(C,s)\vee damaged(C,s))\wedge at(C,J,s),hack(C)]\wedge Poss(hack(C),S_2),S_2\rangle,\]
which can be simplified to $\langle at(C,J,s),S_2\rangle$, and meets the achievement condition via $drive(C,I,J)$ executed in situation $S_0$, and the analysis terminates. 
Once again, note that while not obvious, $drive(C,I,J)$ indeed contributes to $C$'s subsequent damage as it makes the preconditions of $drive(C,J,K)$ true, which in turn makes that of $turn(C,K)$ true, whose execution damages the car. 
\par
Thus, the causal chain obtained is as follows: $\{(turn(C,K),S_4),(drive(C,J,K),S_3),(hack(C),S_2),$ $(drive(C,I,J),S_0)\}.$ Note that, since by Axioms $(g)$, $(h)$, and $(i)$, the initial situation 
is completely specified, it can be shown that this causal chain is unique, i.e.\ there are no other causal chains relative to causal setting $\langle\mathcal{D}_{ac},\phi_1(s),\sigma_1\rangle.$\footnote{As 
mentioned above, we showed this uniqueness property earlier in \cite{KhanS18}.}
%$$$$$$$$$$$$$$ 
\par
Note that, in the above example, not all actions from the trace are included in the causal chain, e.g.\ $turn(C,J)$. To see another example of this, consider the narrative/trace $\sigma_2=do(\vec{a},S_0),$ where: 
\[\vec{a}=[drive(C,I,J), turn(C,J), hack(C), hack(C), drive(C,J,K), turn(C,K), turn(C,K), drive(C,K,J)].\] 
Consider the causal setting $\langle\phi_1,\sigma_2\rangle$. We can show that by Definition \ref{oldd2}, the second $hack(C)$ action executed in $S_3=do([drive(C,I,J),turn(C,J),hack(C)],S_0)$ is not a cause 
for this causal setting, nor part of the causal chain relative to this causal setting, since it was preempted by the first $hack(C)$ action. Also, the last two actions, i.e.\ $turn(C,K)$ executed in 
$S_6=do([hack(C), drive(C,J,K), turn(C,K)],S_3)$ and $drive(C,K,J)$ executed in $S_7=do(turn(C,K),S_6)$ 
are irrelevant, since they do not contribute to achieving the effect. In general, there might be several irrelevant actions in between the actions included in the causal chain. It is important 
to realize that our definition can clearly distinguish between irrelevant actions and actions in the causal chain.
\par
We can also handle quantified queries. Consider another example, where we have two agents/cars $C_1$ and $C_2$. 
%Assume that all actions and fluents now have an additional argument, the first one, to denote the car (e.g.\ $turn(C_1,K)$ refers to agent $C_1$'s action of turning the car at intersection $K$). Also assume that all cars are initially at intersection $I$, and are uncorrupted and undamaged. 
We want to determine the actual causes of $\phi=\exists c,c'(c\neq c'\wedge damaged(c,s)\wedge damaged(c',s))$ after each car is hacked and turned, along with some unnecessary actions, starting in situation $S_0$, say in the narrative 
$\sigma_3=do([hack(C_1),turn(C_1,I),$ $hack(C_1),drive(C_2,I,J),turn(C_1,I),hack(C_2),turn(C_2,J),$ $drive(C_1,I,J)],S_0).$
In this case, a similar analysis as above can be used to show that according to our definition the achievement causal chain for this example is as follows: 
$[(turn(C_2,J),S_6),(hack(C_2),S_5),$ $(drive(C_2,I,J),S_3),$ $(turn(C_1,I),S_1),(hack(C_1),S_0)],$ 
where $S_1=do(hack(C_1),S_0)$, $S_2=do(turn(C_1,I),$ $S_1),$ etc. 
%
%~~~~~~~~~~~~~~~~~~~~~~~~~~~~~~~~~~~~~~~~~~~~~~~~~~~~~~~~~~~~~~~~~~~~~~~~~~~~~~~~~~~~~~~~~~~~~~~~~~~~~~~~~~~~~~~~~~~~~~~~~~~~~~~~~~~~~
%~~~~~~~~~~~~~~~~~~~~~~~~~~~~~~~~~~~~~~~~~~~~~~~~~~~~~~~~~~~~~~~~~~~~~~~~~~~~~~~~~~~~~~~~~~~~~~~~~~~~~~~~~~~~~~~~~~~~~~~~~~~~~~~~~~~~~
%~~~~~~~~~~~~~~~~~~~~~~~~~~~~~~~~~~~~~~~~~~~~~~~~~~~~~~~~~~~~~~~~~~~~~~~~~~~~~~~~~~~~~~~~~~~~~~~~~~~~~~~~~~~~~~~~~~~~~~~~~~~~~~~~~~~~~
%~~~~~~~~~~~~~~~~~~
%~~~~~~~~~~~~~~~~~~
%~~~~~~~~~~~~~~~~~~
\section{Knowledge in the Situation Calculus}
We now return to our discussion of actual epistemic achievement causes, i.e.\ causes of an effect from the perspective of an agent. To deal with this, we allow the domain specifier 
to model agents' mental states, in particular their knowledge. We start by adapting a simple model of knowledge and knowledge change in the situation calculus in this section. In Section 5, we 
will then extend this notion to handle ``knowing the causes of an effect''. This allows an agent to reason about causes of various effects.   
\par
\subsection*{Knowledge}
Following \cite{Moore85,SchLev03}, we model knowledge using a possible worlds account adapted to the situation calculus. 
To allow for the possibility of incomplete initial knowledge, we can now have multiple initial situations. We use $Init(s)$ to mean that $s$ is an initial situation where no action has happened yet, 
i.e.\ $\neg\exists a,s'.\;s=do(a,s').$ The actual initial situation is denoted by $S_0$. 
Also, $K(agt,s',s)$ is used to denote that in situation $s$, the agent $agt$ thinks that she could be in situation $s'$. 
$s'$ is called a $K$-alternative situation for agent $agt$ in situation $s$. Using $K$, the knowledge of an agent, $Know(agt,\phi,s)$, 
is defined as:\footnote{We will use state formulae within the scope of knowledge. A state formula $\phi(s)$ takes a single situation as argument and is evaluated with respect to that situation. 
We often use $\phi$ to denote a formula whose fluents may contain a placeholder constant $now$ that stands for the situation in which $\phi$ must hold. $\phi(s)$ is the 
formula that results from replacing $now$ with $s$. Where the intended meaning is clear, we sometimes suppress the placeholder.}
\begin{definition}[Knowledge]
$Know(agt,\phi(now),s)\defi\forall s'.\;(K(agt,s',s)\rightarrow\phi(s')).$ 
\end{definition}
That is, an agent $agt$ knows that the formula $\phi$ holds in situation $s$ if $\phi$ holds in all of $agt$'s $K$-accessible situations in $s$. As in \cite{Shapiro05}, who generalized the $Know$ 
and $K$ notation to handle multiple agents by adding an agent argument to them, we adopt this convention; however we will suppress the agent argument when dealing with single agent domains. 
\par
Scherl and Levesque \cite{SchLev03} extended Reiter's successor-state axiom approach to model the effects of actions on agents' knowledge, combining ideas from Reiter and Moore. 
As in \cite{SchLev03}, we require that initial situations can only be $K$-related to other initial situations:
\[\forall agt,s,s'(Init(s)\wedge K(agt,s',s)\rightarrow Init(s')).\]
As we will see later, the successor-state axiom for $K$ ensures that in all the situations that are $K$-accessible from $do(a, s)$, $a$ was the last action performed. 
This along with the above requirement thus implies that all $K$-related situations share the same action history. 
We also constrain $K$ to be reflexive, transitive, and Euclidean in the initial situation to capture the fact that agents' knowledge is true, and that agents have positive and negative introspection:
\begin{eqnarray*}
&&\forall agt,s(Init(s)\rightarrow K(agt,s,s)),\\
&&\forall s(Init(s)\rightarrow\forall agt,s_1,s_2(K(agt,s_1,s)\wedge K(agt,s_2,s_1)\rightarrow K(agt,s_2,s))),\\
&&\forall s(Init(s)\rightarrow\forall agt,s_1,s_2(K(agt,s_1,s)\wedge K(agt,s_2,s)\rightarrow K(agt,s_2,s_1))).
\end{eqnarray*} 
As shown in \cite{SchLev03}, these constraints then continue to hold after any sequence of actions since they are preserved by the successor state axiom for $K$. 
%~~~~~~~~~~~~~~~~~~
%~~~~~~~~~~~~~~~~~~
%~~~~~~~~~~~~~~~~~~
\subsection*{Example (Continued)}
We want to model an agent's knowledge --both about the world and about the actual achievement causes of effects-- in the above autonomous vehicle domain. Assume that the agent initially knows that the car $C$ is undamaged and that $C$ is located at intersection $I$:
\begin{eqnarray*}
&&(o).\;Know(\neg damaged(C,now),S_0),\hspace{10 mm}(p).\;Know(\forall i.\;at(C,i,now)\leftrightarrow i=I,S_0).
\end{eqnarray*} 
Thus $\neg damaged(C)\wedge at(C,I)$ holds in all of her initial $K$-accessible worlds/situations. 
Also assume that the agent does not know anything about the integrity of $C$'s T-CAS: 
\[(q).\;\neg Know(corrupted(C,now),S_0)\wedge\neg Know(\neg corrupted(C,now),S_0).\]
Thus, initially there are at least two possible worlds that are $K$-related to the initial situation 
$S_0$, say $S_0$ and $S_0^*$ (this is depicted in Figure \ref{fig:evo0}). 
Each of these worlds assigns a different interpretation to the corruptedness of the car's T-CAS.  
%~~~~~~~~~~~~~~~~~~
%~~~~~~~~~~~~~~~~~~
%~~~~~~~~~~~~~~~~~~
\begin{comment}
\begin{figure}[t!]%[h]
%\setlength{\unitlength}{0.14in} % selecting unit length
%\centering % used for centering Figure
\begin{center}
\scalebox{0.6}{
\tikzset{every loop/.style={min distance=15mm,looseness=10}}
\begin{tikzpicture}[-latex ,auto ,node distance =2.6cm and 5cm, on grid,semithick ,
state/.style ={circle, draw, color=black, fill=black, text=white , minimum width =0.2 cm}]

\node[state] (s0) [label=below left:$S_0$]{};
\node[state] (s01) [above left=of s0][label=above:{$S_0^1$}]{};
\node[state] (s02) [above =of s0][label=above:{$S_0^2$}]{};
\node[state] (s03) [above right=of s0][label=above:{$S_0^3$}]{};

\draw (s0) -- (s01) node [midway, above] (TextNode21a) {$K$};
\draw (s0) -- (s02) node [midway, right] (TextNode21b) {$K$};
\draw (s0) -- (s03) node [midway, above] (TextNode21c) {$K$};
%\path (s0) edge [loop below] node {$K$} (s0);
\Loop[dist=2cm,dir=SOEA,label=$K$,labelstyle=below right](s0);

%%%%%

\end{tikzpicture}
}%scalebox
\end{center}
\caption{Initial knowledge of our blocks world agent in situation $S_0$} % title of the Figure
\label{fig3a} % label to refer figure in text
\end{figure}
\end{comment}
%~~~~~~~~~~~~~~~~~~
%~~~~~~~~~~~~~~~~~~
%~~~~~~~~~~~~~~~~~~
\subsection*{Knowledge Change}
Scherl and Levesque \cite{SchLev03} showed how to capture the changes in knowledge of agents that result from actions in the successor state axiom for $K$. 
These include knowledge-producing actions that can be either binary sensing actions or non-binary sensing actions. 
A binary sensing action is a sensing action that senses the truth-value of an associated proposition; e.g., the binary sensing action $sense_{isCorrupted}(agt)$ could be performed to sense whether 
the agent/car $agt$'s T-CAS is corrupted or not. On the other hand, non-binary sensing actions refer to sensing actions where the agent senses the value of an associated term; e.g., the hypothetical non-binary sensing action  $computePercentageOfDamage(agt)$ could be performed to get the percentage of damage to the agent $agt$. Following \cite{Lev96}, the information 
provided by a binary sensing action is specified using the predicate $SF(a, s)$, which holds if the action $a$ returns the binary sensing result 1 in situation $s$. 
%Similarly for non-binary sensing actions, the term $sff(a, s)$ is used to denote the sensing value returned by the action. 
A \emph{guarded sensed fluent axiom} is used to associate an action with the 
property sensed by this action. For example, one might have a guarded sensed fluent axiom to assert that the action $sense_{isCorrupted}(c)$ tells the agent $c$ whether her T-CAS is corrupted in the situation where it is performed, provided that $c$ is located at the garage:
\[at(c,Garage)\rightarrow (SF(sense_{isCorrupted}(c),s) \leftrightarrow corrupted(c,s)).\]
Similarly for non-binary sensing actions, the term $sff(a, s)$ is used to denote the sensing value returned by the action. For example, the following guarded sensed fluent axiom asserts that the action  $computePercentageOfDamage(c)$ tells $c$ the percentage of damage to the car, provided that $c$ is at the garage:
\[at(c,Garage) \rightarrow (sff(computePercentageOfDamage(c), s) = percentageOfDamageOn(c,s)).\]
\par
The successor-state axiom for $K$ is defined as follows:\footnote{Lesp\'{e}rance \cite{Lesperance03} and later others \cite{Shapiro05} have extended the successor-state axiom for $K$ to support different types of 
communication actions, but for simplicity we do not consider communication actions here.}
\begin{axiom}[Successor-State Axiom for $K$]
\begin{eqnarray*}
&&\hspace{-9 mm}K(agt,s^*,do(a, s))\leftrightarrow\mbox{}\\
&&\hspace{5 mm}\exists s'\;. [K(agt, s',s)\wedge s^*= do(a, s')\wedge Poss(a,s')\\
&&\hspace{13 mm}\mbox{}\wedge((BinarySensingAction(a)\wedge Agent(a) = agt)\rightarrow(SF(a,s')\leftrightarrow SF(a,s)))\\
&&\hspace{13 mm}\mbox{}\wedge((NonBinarySensingAction(a)\wedge Agent(a) = agt)\rightarrow(sff(a, s') = sff(a, s)))].
\end{eqnarray*}
\end{axiom}
\noindent This says that after an action happens, every agent learns that it has happened. Thus, an agent's knowledge is affected by every action in the sense that 
she comes to know that the action was performed. It is assumed that agents know the successor-state axioms for actions, so the agents also acquire knowledge about the effects 
of these actions.\footnote{One consequence of this is that agents are assumed to be aware of all actions that may happen in the environment. This in part allows us to avoid belief revision 
and its difficulties.} 
%
%
%~~~~~~~~~~~~~~~~~~
%~~~~~~~~~~~~~~~~~~444
%~~~~~~~~~~~~~~~~~~
%\begin{comment}
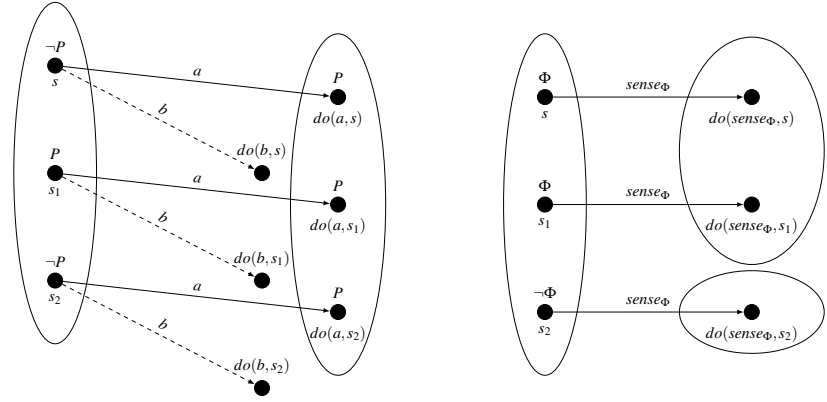
\begin{figure}[t!]%[h]
%\setlength{\unitlength}{0.14in} % selecting unit length
%\centering % used for centering Figure
\begin{center}
\scalebox{0.55}{%\scalebox{0.8}{
\tikzset{every loop/.style={min distance=15mm,looseness=10}}
\begin{tikzpicture}[-latex ,auto ,node distance =2.6cm and 5cm, on grid,semithick ,
state/.style ={circle, draw, color=black, fill=black, text=white , minimum width =0.2 cm}]

\node[state] (a11) [label=below:$s$][label=above:$\neg P$]{};
\node[state] (a12) [below =of a11][label=below:$s_1$][label=above:$P$]{};
\node[state] (a13) [below =of a12][label=below:$s_2$][label=above:$\neg P$]{};
\node[fit={(a11)(a13)},draw, ellipse,minimum width=2cm](left){};

\node[state] (b21) [below right=of a11][label=above:{$do(b,s)$}]{};
\node[state] (b22) [below =of b21][label=above:{$do(b,s_1)$}]{};
\node[state] (b23) [below =of b22][label=above:{$do(b,s_2)$}]{};

\draw [dashed] (a11) -- (b21) node [midway, above, sloped] (TextNode21) {$b$};
\draw [dashed] (a12) -- (b22) node [midway, above, sloped] (TextNode22) {$b$};
\draw [dashed] (a13) -- (b23) node [midway, above, sloped] (TextNode23) {$b$};

\node[state] (b11) [above right of =b21][label=below:{$do(a,s)$}][label=above:$P$]{};
\node[state] (b12) [below =of b11][label=below:{$do(a,s_1)$}][label=above:$P$]{};
\node[state] (b13) [below =of b12][label=below:{$do(a,s_2)$}][label=above:$P$]{};
\node[fit={(b11)(b13)},draw, ellipse,minimum width=2.3cm](left){};

\draw (a11) -- (b11) node [midway, above, sloped] (TextNode11) {$a$};
\draw (a12) -- (b12) node [midway, above, sloped] (TextNode12) {$a$};
\draw (a13) -- (b13) node [midway, above, sloped] (TextNode13) {$a$};

%%%%%

\node[state] (a1) [right =of b11][label=below:$s$][label=above:$\Phi$]{};
\node[state] (a2) [below =of a1][label=below:$s_1$][label=above:$\Phi$]{};
\node[state] (a3) [below =of a2][label=below:$s_2$][label=above:$\neg\Phi$]{};
% left ellipse
\node[fit={(a1)(a3)},draw, ellipse,minimum width=2cm](left){};

\node[state] (b1) [right =of a1][label=below:{$do(sense_\Phi,s)$}]{};
\node[state] (b2) [below =of b1][label=below:{$do(sense_\Phi,s_1)$}]{};
\node[state] (b3) [below =of b2][label=below:{$do(sense_\Phi,s_2)$}]{};
% right ellipse
\node[fit={(b1)(b2)},draw, ellipse,minimum width=3.5cm,minimum height=5.5cm](right){};
\node[fit={(b3)},draw, ellipse,minimum width=3.5cm,minimum height=2cm](right){};

\draw (a1) -- (b1) node [midway, above, sloped] (TextNode1) {$sense_\Phi$};
\draw (a2) -- (b2) node [midway, above, sloped] (TextNode2) {$sense_\Phi$};
\draw (a3) -- (b3) node [midway, above, sloped] (TextNode3) {$sense_\Phi$};
%\path (a3) edge (b3); 
%\path [dashed] (an) edge (bn);

\end{tikzpicture}
}%scalebox
\end{center}
\caption{An example of knowledge change} % title of the Figure
\label{fig3} % label to refer figure in text
\end{figure}
%\end{comment}
%~~~~~~~~~~~~~~~~~~
%~~~~~~~~~~~~~~~~~~
%~~~~~~~~~~~~~~~~~~
Moreover, if the action is a sensing action, the agent performing it acquires knowledge of the associated proposition or term. 
%Furthermore, if the action involves someone informing agt that φ holds, then agt knows this afterwards, and similarly for informWhether and informRef. 
Note that this axiom only handles knowledge expansion, not revision.
\par %555
We illustrate the successor-state axiom for $K$ using the scenario in Figure \ref{fig3}. In this figure,  situations are nodes in the graph, and the edges are labeled by actions. Part of the $K$-relation is  represented by the ovals around the nodes. If a situation $s$ appears in the same oval as another situation  $s'$, then $K(agt, s', s)$. Finally, in this figure $s$ denotes the actual situation, i.e.\ the one representing  the true state of the world. First, consider the case for knowledge expansion due to regular actions, as  depicted in the left-hand side of Figure \ref{fig3}. Assume that initially $s$, $s_1$, and $s_2$ are $K$-accessible from  each other. Then after action $a$ happens in situation $s$, according to the successor-state axiom for $K$,  only $do(a, s)$, $do(a, s_1)$, and $do(a, s_2)$ will be accessible from $do(a, s)$, but not $do(b, s_2)$, etc. Thus, in $do(a, s)$ the agent knows that the action $a$ has just happened and knows that its effects hold. If $a$ makes some property $P$ become true in all $K$-accessible situations, then the agent knows that $P$ holds afterwards. Next,  consider the case for knowledge expansion as a result of knowledge producing actions, as illustrated in  the right-hand side of Figure \ref{fig3}. Assume that initially $s$, $s_1$, and $s_2$ are in the same equivalence class wrt $K$, and that $\Phi(s)$, $\Phi(s_1)$, and $\neg\Phi(s_2)$ holds. Then  after the agent senses the value of $\Phi$ in $s$, according to the successor-state axiom for $K$, only  $do(sense_\Phi, s)$ and $do(sense_\Phi, s_1)$ will be $K$-accessible from $do(sense_\Phi, s)$, but not  $do(sense_\Phi, s_2)$. Since $\Phi$ holds in all situations that are $K$-accessible from $do(sense_\Phi, s)$, the agent will thus know that $\Phi$ in $do(sense_\Phi, s)$.
%
%~~~~~~~~~~~~~~~~~~
%~~~~~~~~~~~~~~~~~~
%~~~~~~~~~~~~~~~~~~
\section{Actual Epistemic Achievement Causes}
%
%\subsection*{Epistemic Causes}
We next formalize a notion of knowledge relative to a causal setting. While obvious, we would like to remind the reader that given a 
causal setting $\mathcal{C}=\langle\mathcal{D},do([\alpha_1,\cdots,\alpha_n],\sigma),\phi(s)\rangle$ and a causal chain $\mathcal{K}=\{(a_1,s_1)\cdots,(a_m,s_m)\}$ of $\mathcal{C},$ the actions $a_1,\cdots,a_m$ 
in $\mathcal{K}$ must come from the trace $\alpha_1,\cdots,\alpha_n$. We start by defining the following concept of $K$-related causal chains.
\begin{definition}\label{new01}
Consider two causal settings $\mathcal{C}_1=\langle\mathcal{D},do([\alpha_1,\cdots,\alpha_n],\sigma_1),\phi(s)\rangle$ and 
$\mathcal{C}_2=\langle\mathcal{D},do([\alpha_1,\cdots,$ $\alpha_n],\sigma_2),\phi(s)\rangle$ that differ only in the situations where their narratives start, i.e.\ in initial situations $\sigma_1$ and $\sigma_2$. 
Assume that $\mathcal{K}_1$ is a (non-empty) achievement causal chain of causal setting $\mathcal{C}_1$, and $\mathcal{K}_2$ that of $\mathcal{C}_2$. 
We say that $\mathcal{K}_1$ and $\mathcal{K}_2$ are $K$-related with respect to the achievement of $\phi(s)$ and action sequence $[\alpha_1,\cdots,\alpha_n]$  
if and only if $\mathcal{K}_1$ is of the form $\{(a_1,s_1^1),\cdots,(a_m,s_m^1)\}$ and $\mathcal{K}_2$ is of the form $\{(a_1,s_1^2),$ $\cdots,(a_m,$ $s_m^2)\}$ 
for some $m>0$, and for all $1\leq i\leq m$, it follows that $\mathcal{D}\models K(agt,s_i^1,s_i^2).$ 
\end{definition}
\noindent 
Thus, two causal chains are $K$-related if they have the same cardinality (i.e.\ equal number of (action,situation) pairs), and for every (action, situation) pairs in these causal chains, 
they only (possibly) differ in the situation term, which are required to be $K$-related. Note that since $K$ is reflexive, this holds trivially when $\mathcal{K}_1=\mathcal{K}_2.$ 
Intuitively, if two causal chains $\mathcal{K}_1$ and $\mathcal{K}_2$ are $K$-related, then as far as the agent is concerned, there is no difference between these two causal chains 
relative to the achievement of the effect $\phi$ via the execution of the sequence of events $[\alpha_1,\cdots,\alpha_n].$ 
\par
Using this, we define the knowledge of a causal chain relative to a causal setting as follows:
\begin{definition}\label{new02}
Given a causal setting $\mathcal{C}=\langle\mathcal{D},\sigma,\phi(s)\rangle,$ where $\sigma=do([\alpha_1,\cdots,\alpha_n],s^*)$ for some initial situation $s^*$ and finite $n>0$, 
an agent knows in situation $\sigma$ that $\mathcal{K}$ is the achievement causal chain of $\mathcal{C}$ if and only if: 
\begin{itemize}
\item $\mathcal{K}$ is an achievement causal chain of $\mathcal{C}$, and 
\item for all $\sigma^*$ such that $\mathcal{D}\models K(\sigma^*,\sigma),$ if $\mathcal{K}^*$ is an achievement causal chain of 
causal setting $\langle\mathcal{D},\sigma^*,\phi(s)\rangle,$ then causal chains $\mathcal{K}$ and $\mathcal{K}^*$ are $K$-related relative to the achievement of $\phi(s)$ and 
the action sequence $[\alpha_1,\cdots,\alpha_n]$.   
\end{itemize}
\end{definition}
\noindent 
That is, an agent knows in situation $\sigma$ that $\mathcal{K}$ is the achievement causal chain of $\mathcal{C}$ if and only if each of the causal chains computed in the worlds that the agent 
considers possible in $\sigma$ is $K$-related wrt the effect and the trace in $\mathcal{C}$. In the following, we will use the term \emph{actual narrative} to refer to any ground situation 
$\sigma=do([\alpha_1,\cdots,\alpha_n],s)$ for some $n>0,$ where $s$ is the actual initial situation, i.e.\ $s=S_0.$  
\par
Note that, Definition \ref{new02} implicitly specifies that if an agent knows in situation $\sigma$ that $\mathcal{K}$ is the achievement causal chain of the causal setting $\mathcal{C}$, then the causal 
chain of setting $\mathcal{C}$ is unique. Since $K$ is reflexive, any causal chain relative to setting $\mathcal{C}$ must be $K$-related wrt the achievement of $\phi$ and trace $[\alpha_1,\cdots,\alpha_n].$ 
Since the setting $\mathcal{C}$ (and as such the situation $\sigma$) does not change, this implies the uniqueness of $\mathcal{K}.$ 
By the same token, when an agent has the knowledge of a causal chain relative to a causal setting in some situation $\sigma$, the causal chain obtained in each of the $K$-alternative situations to 
$\sigma$ (relative to the respective causal settings) is also unique. 
\par
Thus, the above definition specifies the conditions under which an agent can be said to know the actual causes of an observed effect. 
Note that, according to our definition of epistemic causality, it is possible for an agent to not know the causes of a known effect. 
For example, in decentralized voting protocols, a system/agent that cast a decisive majority is the actual cause of the decision, but another system/agent may not know this since each ballot was 
secretly cast. In such cases, reasoning must be performed on different epistemic alternatives separately using the original notion of causality.
\par
To appreciate the power of our formalization, note that equipped with the ability to deal with knowledge of the causes of an effect, agents specified in our framework can now reason about the causes of effects relative to/conditioned on what they know. Also, since agents are introspective relative to their knowledge (i.e.\ they know what they know and know what they don't know), they can also reason about the causes of epistemic effects (changes in their knowledge). Furthermore, incorporating other intentional attitudes within this framework, such as goals and intentions --as was done in \cite{KhanL10})-- allows agents to reason about the causes of change in their motivations (agents are also introspective relative to their motivational attitudes; see \cite{KhanL10} for details). Finally, when multiple agents are involved, this allows agents to reason about each other's knowledge and goals. Therefore, this simple extension unleashes the power of causal analysis and allows agents to reason about the causes of various intentional attitudes. 
%~~~~~~~~~~~~~~~~~~
%~~~~~~~~~~~~~~~~~~%333
%~~~~~~~~~~~~~~~~~~
\begin{figure}[t!]%[h]
\setlength{\unitlength}{0.14in} % selecting unit length
\centering % used for centering Figure
\[
\begin{tikzcd}[row sep=20pt,column sep=huge]%row sep=60pt
S_0\arrow[loop above]{}{K}\arrow{r}{drive(C,I,J)}\arrow{d}{K} 
& S_1\arrow[loop above]{}{K}\arrow{r}{turn(C,J)}\arrow{d}{K} 
& S_2\arrow[loop above]{}{K}\arrow{r}{hack(C)}\arrow{d}{K} 
& S_3\arrow[loop above]{}{K}\arrow{r}{drive(C,J,K)}\arrow{d}{K} 
& S_4\arrow[loop above]{}{K}\arrow{r}{turn(C,K)}\arrow{d}{K} 
& \sigma_1\arrow[loop above]{}{K}\arrow{d}{K}\\
S_0^*\arrow{r}{drive(C,I,J)} 
& S_1^*\arrow{r}{turn(C,J)} 
& S_2^*\arrow{r}{hack(C)} 
& S_3^*\arrow{r}{drive(C,J,K)} 
& S_4^*\arrow{r}{turn(C,K)}
& S_5^*
%\arrow[sloped,swap]{ur}[pos=0.15]{g({\varphi}){\mapsto}{\varphi(G)}}[swap]{\cong}
\end{tikzcd}
\]
\caption{Evolution of the (partial) $K$ relation of the autonomous vehicle agent wrt the actions in $\sigma_1$} % title of the Figure
\label{fig:evo2} % label to refer figure in text
\end{figure}
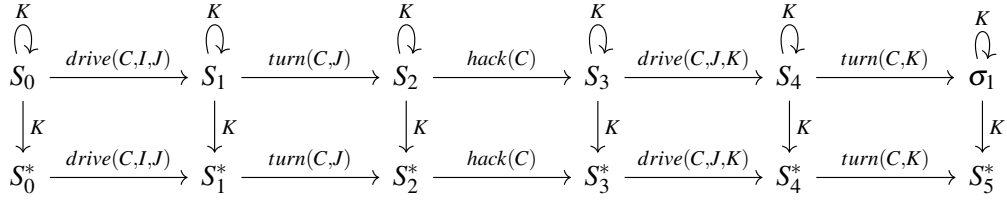
%~~~~~~~~~~~~~~~~~~
%~~~~~~~~~~~~~~~~~~
%~~~~~~~~~~~~~~~~~~
\subsection*{Example (Continued)}
Let $\mathcal{D}_{ac}^K$ refer to $\mathcal{D}_{ac}$ along with the above axiomatization of knowledge, knowledge change, and our agent's (initial) knowledge 
about the autonomous vehicle domain (i.e.\ Axioms $(o)-(q)$). Using Definitions \ref{new01} and \ref{new02}, we can show the following result on epistemic causality 
in our autonomous vehicle domain:
\begin{theorem}
Given causal setting $\mathcal{C}_{ac}=\langle\mathcal{D}_{ac},\sigma_1,\phi_1(s)\rangle$ and its achievement causal chain 
$\mathcal{K}_{ac1}=\{(turn(C,K),S_4),(drive(C,J,K),S_3),(hack(C),S_2),(drive(C,I,J),S_0)\}.$ 
$\mathcal{D}_{ac}^K$ entails that the agent does not know in $\sigma_1$ that $\mathcal{K}_{ac1}$ is the achievement causal chain of causal setting $\mathcal{C}_{ac}.$
 \end{theorem} 
 %yyy
 \begin{proof}
%As we pointed out above, $\mathcal{K}_{bw1}$ is the unique causal chain of causal setting $\mathcal{C}_{bw}$. Thus 
By Definition \ref{new02}, to prove this we need to show that there exists a situation $\sigma^*$ that is $K$-accessible from $\sigma_1$, i.e.\ $K(\sigma^*,\sigma_1)$, and an achievement causal chain 
$\mathcal{K}^*$ of the causal setting $\langle\mathcal{D}_{ac},\sigma^*,\phi_1(s)\rangle$ is not $K$-related to the achievement causal chain $\mathcal{K}_{ac1}$ (according to Definition \ref{new01}).  
Let us consider an initial situation where the agent $C$'s T-CAS is corrupted, $C$ is undamaged, $C$ is located at intersection $I$; let us call this situation $S_0^*$:  
%(see Figure \ref{fig3a}): 
\begin{equation}\label{prf1}corrupted(C,S_0^*)\wedge\neg damaged(C,S_0^*)\wedge at(C,I,S_0^*).\end{equation}
We claim that the successor to this situation after the actions in the situation $\sigma_1$ has happened, i.e.\ situation 
$S_5^*=do([drive(C,I,J),turn(C,J),hack(C),drive(C,J,K),turn(C,K)],S_1^*)$, as can be seen in Figure \ref{fig:evo2}, is indeed 
such a situation $\sigma^*$. To show this, we have to show that $K(S_5^*,\sigma_1)$ and that $\mathcal{K}_{ac1}$ and $\mathcal{K}^*$ are not $K$-related wrt the achievement of $\phi_1(s)$ and the 
action sequence in $\sigma_1.$
\par
We start by showing the former (see Figure \ref{fig:evo2}). Note that it follows from $\mathcal{D}_{ac}^K$ and (\ref{prf1}) that $S_0^*$ is $K$-accessible from the actual initial situation $S_0$, i.e.\ $\mathcal{D}_{ac}^K\models K(S_0^*,S_0).$ Moreover, it can be shown that $\mathcal{D}_{ac}^K$ entails that all the actions in $\sigma_1$ are known to be executable starting in $S_0$. 
Furthermore, since all the actions performed in $\sigma_1$ and $S_5^*$ are exactly the same, and since none of these actions are knowledge-producing/sensing 
actions, by the successor-state axiom for $K$, it follows that $S_5^*$ is retained in the $K$-relation in $\sigma_1$. Thus we have $K(S_5^*,\sigma_1)$.   
\par
Now, computing the achievement causal chain for causal setting $\langle\mathcal{D}_{ac},S_5^*,\phi_1\rangle$ using Definition 2 yields the causal chain 
$\mathcal{K}^*=\{(turn(C,J),S_1^*),(drive(C,I,J),S_0^*)\},$ as can be seen in the bottom part of Figure \ref{fig:evo0}. 
%\footnote{It should be clear that the argument in the fluents and actions in Figure \ref{fig:evo0} is $B_1$.}
By Definition \ref{new01}, $\mathcal{K}_{ac1}$ and $\mathcal{K}^*$ are clearly not $K$-related wrt the achievement of $\phi_1(s)$ and action sequence 
$[drive(C,I,J),turn(C,J),hack(C),drive(C,J,K),turn(C,K)]$. 
\end{proof}   
The above theorem demonstrates that it is possible for the same effect brought about by the same sequence of actions to have different causes in different epistemic alternatives, as is expected. 
Put otherwise, when mental attitudes are concerned, theories of causation at different levels of (epistemic) nestings need not be related. In the words of Williamson 
\cite{Williamson06}, ``To say that causal connection is mental does \emph{not} imply that causality is subjective (in the logical sense)''. 
%
%~~~~~~~~~~~~~~~~~~~~~~~~~~~~~~~~~~~~~~~~~~~~~~~~~~~~~~~~~~~~~~~~~~~~~~~~~~~~~~~~~~~~~~~~~~~~~~~~~~~~~~~~~~~~~~~~~~~~~~~~~~~~~~~~~~~~~
%~~~~~~~~~~~~~~~~~~~~~~~~~~~~~~~~~~~~~~~~~~~~~~~~~~~~~~~~~~~~~~~~~~~~~~~~~~~~~~~~~~~~~~~~~~~~~~~~~~~~~~~~~~~~~~~~~~~~~~~~~~~~~~~~~~~~~
%~~~~~~~~~~~~~~~~~~~~~~~~~~~~~~~~~~~~~~~~~~~~~~~~~~~~~~~~~~~~~~~~~~~~~~~~~~~~~~~~~~~~~~~~~~~~~~~~~~~~~~~~~~~~~~~~~~~~~~~~~~~~~~~~~~~~~
\section{Discussion}
The above notion of (objective) causality \cite{BatusovS18} has motivation that is similar to \cite{Gossler15,Wang15}, who also discuss causal analysis relative to traces. However, \cite{Gossler15,Wang15} 
work with less expressive languages, and unlike us they provide a counterfactual definition of causality. 
%their counterfactual definition of causality suffers from the preemption problem. 
As mentioned earlier and shown in \cite{BatusovS18}, our definition above 
can correctly compute actual causes even for the more problematic examples with early preemption and overdetermination that create serious difficulties for the structural equations-based approach 
developed in \cite{Pearl98,Pearl00,HalpernP05,Halpern16}. 
\par
Based on this formal notion of causality, in this paper we proposed an account of epistemic causality within a formal theory of action. Our account allows agents to have incomplete initial knowledge. 
For instance, in our running example, initially the agent doesn't know whether the car's T-CAS system is corrupted or not. We defined what it means for an agent to know the causes of a known effect. We also showed 
that epistemic causality is a different notion from causality in the sense that given a trace, it is possible to have different causes of the same effect in different epistemic alternatives. Thus, as expected, 
the agent may or may not know the causes of an effect. 
%While doing this, we extended the notion of causality in \cite{BatusovS18} by allowing 
\par
Recently, there has been some work that formalizes causality in an epistemic context. For example, while defining responsibility/blame in legal cases, Chockler et al.\ \cite{ChocklerFKL15} modeled an
 agent's uncertainty of the causal setting using an ``epistemic state'', which is a pair $(K,Pr)$, where $K$ is a set of causal settings and $Pr$ is a probability distribution over $K$. Their model is based on 
 structural equations. We on the other hand define epistemic causality based on the more expressive formalism proposed by Batusov and Soutchanski \cite{BatusovS18}. Moreover, unlike \cite{ChocklerFKL15}, our account incorporates a formal model of domain dynamics and knowledge change. This allows for interesting interplay between causality and knowledge. For instance, in our framework it is possible to specify a domain where the agent does not know the causes of an effect in some situation, but learns them after performing some sensing action. Some of our future work include analyzing such examples as well as defining responsibility and blame using our formalization.% of causality.   
%To the best of our knowledge, this is the only formal account of subjective causality that appeals to a first person's perspective of actual causality. 
\par
Here, we focus on knowledge and do not deal with belief. Traditionally, agents' knowledge is required to be true while agents are allowed to have incorrect beliefs \cite{Hintikka62}. 
Incorporating beliefs yields a more expressive framework, one that allows for causality relative to partially observable actions (and traces). 
Put otherwise, the agent can now consider different actions in different doxastic alternatives/belief-accessible worlds: given some situation $s$ in the actual narrative, as far as the agent is concerned, the action 
that was executed in situation $s$ can be any of those performed in one of her belief-accessible worlds (cf.\ Footnote 6, where we required the agent to know the action that happened in $s$). 
The agent can reason about actual causality (under some similar conditions specified in Definition \ref{new02}), even if she does not know the exact sequence of actions that has been performed 
since the initial situation. 
%Thus, by replacing $K$ with $B$ (which stands for the belief operator) in Definition \ref{new01} and \ref{new02}, we can now reason about some given effect even if the actual sequence of actions in the 
%scenario is only partially known --at least as long as the causal chains obtained from each of the $B$-accessible worlds in the actual scenario are ``$B$-accessible'' from the \emph{actual causal chain} (i.e.\ 
%the causal chain relative to the actual scenario).   
%
Also, more interesting interplay between objective and epistemic causality can now arise. For instance, an agent may perceive her action to be a cause for some effect $\phi$, but 
in reality it was not, since $\phi$ was over-determined due to her incorrect beliefs about the world.\footnote{Recall that unlike knowledge, agents' beliefs are not required to be true.} Similarly, an agent 
may think her action was not a cause, but in reality it was. While we think that much of our formalization can be extended to deal with beliefs, this requires handling belief revision, which complicates the 
framework further. We leave this for future.% work.
\par
Finally, we focus in this paper on deterministic actions only. However, there are several proposals on how one can reason about stochastic actions in the situation calculus, e.g.\  \cite{BacchusHL99,BoutilierRP01,BelleL18}. Dealing with stochastic actions is future work.
%Pearl \shortcite{Pearl98,Pearl00} formalized the Humean counterfactual definition of causation, which posits that saying ``an event A caused an outcome B'' 
%is the same as saying ``if A had not been, then B never had existed''. Improvements of this approach have been subsequently proposed by %others 
%\cite{HalpernP05,Halpern16}. 
%However despite claims to the contrary, these frameworks suffer from the problem of \emph{preemption}, 
%; e.g.\ both the disjunctive and the conjunctive versions of the well-known ``Forest Fire'' example \cite{HalpernP05,Halpern16} are properly handled.
%~~~~~~~~~~~~~~~~~~~~~~~~~~~~~~~~~~~~~~~~~~~~~~~~~~~~~~~~~~~~~~~~~~~~~~~~~~~~~~~~~~~~~~~~~~~~~~~~~~~~~~~~~~~~~~~~~~~~~~~~~~~~~~~~~~~~~
%~~~~~~~~~~~~~~~~~~~~~~~~~~~~~~~~~~~~~~~~~~~~~~~~~~~~~~~~~~~~~~~~~~~~~~~~~~~~~~~~~~~~~~~~~~~~~~~~~~~~~~~~~~~~~~~~~~~~~~~~~~~~~~~~~~~~~
%~~~~~~~~~~~~~~~~~~~~~~~~~~~~~~~~~~~~~~~~~~~~~~~~~~~~~~~~~~~~~~~~~~~~~~~~~~~~~~~~~~~~~~~~~~~~~~~~~~~~~~~~~~~~~~~~~~~~~~~~~~~~~~~~~~~~~
\section*{Acknowledgements}
We thank Yves Lesp\'{e}rance and the anonymous reviewers for useful comments on an earlier version.  This work was supported in part by the National Science and Engineering Research Council of Canada and by the Faculty of Science at Ryerson University. 
%\nocite{*}
\bibliographystyle{eptcs}
\bibliography{crest}
\end{document}